\documentclass[runningheads]{llncs}
\usepackage{graphicx, svg}
\usepackage{amsmath,amssymb,amsfonts,mathtools, upgreek}
\usepackage{booktabs, enumitem,multirow}
\usepackage{tikz, pgfplots}
\usetikzlibrary{patterns,spy}
\usepgfplotslibrary{groupplots}

\pgfplotsset{compat=1.17} 

\providecommand{\varitem}{} 
\makeatletter
\newenvironment{axioms}[1]
 {\renewcommand\varitem[1]{\item[\textbf{#1\arabic{enumi}\rlap{$##1$}.}]%
    \edef\@currentlabel{#1\arabic{enumi}{$##1$}}}%
  \enumerate[label=\textbf{(#1\arabic*)}, ref=#1\arabic*,leftmargin=\widthof{[\textbf{(#10)}]}]}
 {\endenumerate}
\makeatother

\begin{document}
    	\title{When Should You Defend Your Classifier\,? --\\ {\normalsize A Game-theoretical Analysis of Countermeasures against Adversarial Examples}\thanks{All authors are supported by the Austrian Science Fund (FWF) and the Czech Science Foundation (GACR) under grant no. I 4057-N31 (``Game Over Eva(sion)''). Tomas Pevny was additionally supported by Czech Ministry of Education 19-29680L and  by the OP VVV project CZ.02.1.01/0.0/0.0/16\_019/0000765 ``Research Center for Informatics''.}}
	\titlerunning{When Should You Defend Your Classifier\,?}
	%
	
 	\author{Maximilian Samsinger\inst{1} \and
 		Florian Merkle\inst{1}\and
 		Pascal Schöttle\inst{1}\and
 		Tomas Pevny\inst{2}}
 	
 	\authorrunning{M. Samsinger et al.}
 	\institute{Management Center Innsbruck, Universitätsstr. 15,
 		Innsbruck, Austria \email{\{maximilian.samsinger,florian.merkle,pascal.schoettle\}@mci.edu}\and
 		Department of Computers and Engineering,
 		Czech Technical University in Prague, Czech Republic
 		\email{pevnak@protonmail.ch}}
	\maketitle              
	\begin{abstract}
	Adversarial machine learning, i.e., increasing the robustness of machine learning algorithms against so-called adversarial examples, is now an established field. Yet, newly proposed methods are evaluated and compared under unrealistic scenarios where costs for adversary and defender are not considered and either all samples or no samples are adversarially perturbed. We scrutinize these assumptions and propose the advanced adversarial classification game, which incorporates all relevant parameters of an adversary and a defender. Especially, we take into account economic factors on both sides and the fact that all so far proposed countermeasures against adversarial examples reduce accuracy on benign samples. Analyzing the scenario in detail, where both players have two pure strategies, we identify all best responses and conclude that in practical settings, the most influential factor might be the maximum amount of adversarial examples.
		\keywords{Adversarial classification \and Game theory \and Correct classification rate.}
	\end{abstract}
\newcommand{\acc}{\ensuremath{\mathrm{acc}}}
\newcommand{\rob}{\ensuremath{\mathrm{rob}}}
\newcommand{\CCR}{\ensuremath{\mathrm{CCR}}}
\newcommand{\ASR}{\ensuremath{\mathrm{ASR}}}
\newcommand{\InitialCosts}{\ensuremath{\mathrm{I}}}
\newcommand{\OngoingCosts}{\ensuremath{\mathrm{O}}}

\newcommand{\Utility}{\ensuremath{\mathrm{Utility}}}
\newcommand{\EPPS}{\ensuremath{\mathrm{EPPS}}}
\newcommand{\ADV}{\ensuremath{\mathrm{adv}}}
\newcommand{\DEF}{\ensuremath{\mathrm{def}\phantom{.}}\!}
\newcommand{\rmax}{\ensuremath{r_{\max}}}

\newcommand{\StrategySetADV}{\ensuremath{\mathcal{R}}}
\newcommand{\StrategySetDEF}{\ensuremath{\mathcal{S}}}

\newcommand{\Dacc}{\ensuremath{\Delta\acc}}
\newcommand{\Drob}{\ensuremath{\Delta\rob}}
\newcommand{\DCCR}{\ensuremath{\Delta\CCR(r;\rmax)}}

\newcommand{\condADV}{\ensuremath{\frac{\rob_2 -1 + \mu^\ADV}{\Drob}}}
\newcommand{\condDEF}{\ensuremath{\frac{\Dacc-\Delta\mu^\DEF}{\Dacc + \Drob}}}

\newcommand{\precondDEF}{\ensuremath{\frac{\Dacc - \Delta\mu^\DEF}{\Dacc + \Drob}}}

\newcommand{\conv}{\ensuremath{\mathrm{\textbf{conv}}}}
\newcommand*{\defeq}{\mathrel{\vcenter{\baselineskip0.5ex \lineskiplimit0pt
			\hbox{\scriptsize.}\hbox{\scriptsize.}}}%
	=}
\newcommand*{\eqdef}{=\mathrel{\vcenter{\baselineskip0.5ex \lineskiplimit0pt
		\hbox{\scriptsize.}\hbox{\scriptsize.}}}%
}

\section{Introduction}

Machine learning and especially deep convolutional neural networks have become the de facto standard in a variety of computer vision related tasks~\cite{he2016deep,krizhevsky2012imagenet}. However, at least since 2004 the vulnerability of machine learning classifiers to carefully crafted attack points is known~\cite{dalvi2004adversarial} and has gained a lot of attention in the research community lately~\cite{biggio2018wild}.
Generally, attacks against machine learning classifiers can be divided into poisoning and evasion attacks~\cite{biggio2018wild}. In the former, the adversary can already tamper with the training procedure, and in the latter, the adversary tries to evade the classifier at inference time. We deal with evasion attacks only in this paper, as especially the vulnerability of neural networks to so-called adversarial examples~\cite{Szegedy13} caused a stir in the machine learning community.
Adversarial examples are benign input points to which tiny, maliciously crafted perturbations are added. For humans, these perturbed images are often indistinguishable from their unmodified counterparts~\cite{madry2017towards} but they trick a neural network into misclassification.
In the light of modern applications, including, but not limited to, facial recognition, self-driving cars, or spam filtering, adversarial examples pose an obvious security concern.

Yet, to the best of our knowledge, most of the papers on either new attack methods or new countermeasures deal with unrealistic conditions: the cost of creating adversarial examples is disregarded; the cost of training or defending the model is not considered; and the proportion of adversarial examples to be expected, is either 100\% or 0\%, but nothing in between. Taking into account the observation that all countermeasures against adversarial examples proposed so far strictly decrease the accuracy on benign images, we argue that depending on the expected amount of adversarial examples faced, defending a neural network classifier by one of these countermeasures might not be worth it. As a motivating example, we use the clean and robust accuracies from the famous paper~\cite{madry2017towards} to create Figure~\ref{fig:Madry}. Here, we show the expected proportion of adversarial examples ($x$-axis) and an evaluation metric (formally introduced in Sec.~\ref{sec:formalization}) that balances the classifier's adversarial robustness against its accuracy on benign inputs ($y$-axis).

This shows that when expecting less then 17\% adversarial examples, the undefended model (solid line) yields a higher correct classification rate than the defended model (dashed line) and thus, is favorable for the defender -- assuming that misclassifying a benign or and adversarial example induce the same penalty.


\begin{figure}[t]
	\centering
	\newcommand{\strength}{20}
\begin{tikzpicture}

\definecolor{color0}{rgb}{0.12156862745098,0.466666666666667,0.705882352941177}
\definecolor{color1}{rgb}{1,0.498039215686275,0.0549019607843137}

\begin{axis}[
width=.75\columnwidth,
height=.45\columnwidth,
legend cell align={left},
legend style={fill opacity=0.8, draw opacity=1, text opacity=1, at={(0.03,0.03)}, anchor=south west, draw=white!80!black},
tick align=outside,
tick pos=left,
x grid style={white!69.0196078431373!black},
xlabel={Adversary's Strategy (attacking probability)},
xmajorgrids,
xmin=-0.05, xmax=1.05,
xtick style={color=black},
y grid style={white!69.0196078431373!black},
ylabel={Defender's CCR (see Def.~\ref{def:CCRij})},
ymajorgrids,
ymin=-0.04, ymax=1.04,
ytick style={color=black}
]
\addplot [thick, color0]
table {%
0 0.952
1 0.035
};
\addlegendentry{Standard training}
\addplot [thick, color0, dashed]
table {%
0 0.873
1 0.458
};
\addlegendentry{Adversarial training}
\end{axis}

\end{tikzpicture}
	\caption{Correct classification rate on CIFAR-10 test images based on the original data in \cite{madry2017towards}, as mentioned in Sec.~\ref{sec:relatedwork}. We consider the ``wide'' architecture.}
	\label{fig:Madry}
\end{figure}
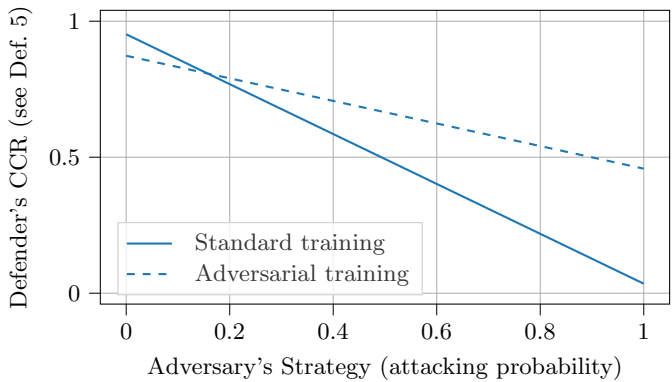
Given this drawback of all so far proposed countermeasures against adversarial examples, we raise the question if a rationally acting defender would actually deploy these countermeasures when facing a rationally acting adversary. 
We do so by proposing a game-theoretical model and a thorough analysis of a minimal instantiation.
Our contributions are as follows:
\begin{enumerate}
    \item We propose the \emph{Advanced Adversarial Classification Game} that captures all relevant properties in the competition between adversary and defender in adversarial machine learning. Our game can be instantiated with all possible defender and adversary strategies (meaning classification models on the defender's side and attack algorithms on the adversary's side).
    \item We thoroughly analyse the game and identify situations where both players play pure (or mixed) strategies.
    \item We define two new metrics, the correct classification rate (CCR) for the defender and the attack success rate (ASR) for the adversary, and show their role in the analysis of the game.
    \item By starting with a rigorous mathematical formulation of an economic model and identifying all simplifications made, we justify the sufficiency and practical importance of CCR and ASR.
\end{enumerate}

The remainder of the paper is structured as follows: In Sec.~\ref{sec:relatedwork} we review the most relevant related work before introducing our game-theoretical model in Sec.~\ref{sec:general}. We instantiate this general model in Sec.~\ref{sec:analysis} and analyze the case where both, defender and adversary have two possible strategies. We report best responses and Nash equilibria before ending with a discussion and conclusion.

\section{Related Work}\label{sec:relatedwork}
In the last couple of years, stronger and stronger attacks against machine learning classifiers were developed, e.g.,~\cite{brendel2017decision,papernot2016transferability,papernot2017practical}, 
with basically everyday new papers appearing on this subject\footnote{More than 3500 papers can be found at: \url{https://preview.tinyurl.com/yxenrc4k}}.

Naturally, with an increased interest in attacks, also an increasing number of countermeasures have been proposed. Already the seminal work that first identified the existence of adversarial examples~\cite{Szegedy13} suggested to harden the underlying convolutional neural networks (CNN) against these adversarial examples by incorporating them into the training procedure. This method is by now commonly called \emph{adversarial training}. It is one of the most prominent approaches in the research of CNNs that are robust against adversarial examples~\cite{madry2017towards}. Another approach to avert the danger of adversarial examples is to try to detect them either inside the CNN itself~\cite{Grosse17}, or by applying detection methods to every input object and sorting out adversarial examples before they even enter the neural network~\cite{Xu18}. Here, one important observation is that all countermeasures against adversarial examples proposed so far strictly decrease the accuracy on benign images. For example, adversarial training decreases clean accuracy from $95.2\%$ to $87.3\%$ in the ``wide'' model considered in~\cite{madry2017towards} and from $95.6\%$ to $90.0\%$ with the approach proposed in~\cite{zhang2019you}. Even the randomized defence mechanism proposed in~\cite{pinot2020randomization} decreases the accuracy on benign inputs from $88\%$ to $80\%$. (All results are for CIFAR-10.)

Most of the literature on attacks against machine learning classifiers and countermeasures against these assumes that an adversary will always attack and a defender will always (try to) defend, irrespective of whether it actually pays off to do so. As soon as we assume both, adversary and defender to act rationally, and only attack and defend when it pays off, we enter the realm of game theory.  
Game-theoretical analysis of adversarial machine learning dates back to 2004, when Dalvi et al. analyzed the security of a machine learning-based spam detector against a strategic adversary~\cite{dalvi2004adversarial}. Here, the spam detector is a binary classifier and the adversary creates adversarial examples by perturbing some well chosen features of legitimate emails. The simultaneous move game is solved for a Nash equilibrium~\cite{nash1951non} by formulating the problem as a constrained optimization problem and solving this with a mixed linear program. 

Following up on this, researchers have framed \emph{adversarial classification} with zero-sum~\cite{globerson2006nightmare} and non zero-sum games~\cite{dritsoula2017game} and as simultaneous move and sequential move (Stackelberg) games~\cite{bruckner2011stackelberg}. With the increasing interest of the machine learning community in deep neural networks, also the game-theoretical analysis of these machine learning frameworks~\cite{schuurmans2016deep}, as well as their security properties~\cite{grosshans2015solving,pinot2020randomization} appeared. 
Interestingly, to the best of our knowledge, there is no work which that to incorporate costs on both the adversary's and the defender's side and the decreased accuracy on benign inputs, mentioned above, and then analyzes the optimal strategies of both players.

Probably closest to our work is Gilmer et al.'s paper titled ``Motivating the Rules of the Game for Adversarial Example Research''~\cite{gilmer2018}. Besides the title, it does not deal with game theory although many of its arguments follows game-theoretical considerations, such as the question who moves first, the adversary or the defender (like in a Stackelberg game), or allowing the adversary a strategy between never attacking and always attacking (as in a mixed strategy).

We adapt the following definitions from game theory literature, e.g.,~\cite{leyton2008essentials}:

\begin{definition}[Mixed strategy]
A \textbf{mixed strategy} is a strategy, which assigns a positive probability to two or more pure strategies.
A \textbf{fully} mixed strategy assigns a positive probability to all pure strategies. 
\end{definition}

\begin{definition}[Best response]
A defender's \textbf{best response} $s^* \in \StrategySetDEF$ to an adversary's strategy $r \in \StrategySetADV$ satisfies $\Utility^\DEF\!(s^*, r) \geq \Utility^\DEF\!(s, r)$ for all $s \in \StrategySetDEF$.
Likewise, an adversary's \textbf{best response} $r^* \in \StrategySetADV$ to a defender's strategy $s\in \StrategySetDEF$ satisfies $\Utility^\ADV\!(s, r^*) \geq \Utility^\ADV\!(s, r)$ for all $r \in \StrategySetADV$.
\end{definition}

\begin{definition}[Nash equilibrium]
A \textbf{Nash equilibrium} is a strategy profile ($s^*,r^*$) where both strategies are mutual best responses.
In a Nash equilibrium neither of the two actors has an incentive to unilaterally change her strategy.






\noindent A game can have zero, one, or multiple Nash equilibria.
\end{definition}

\section{The Advanced Adversarial Classification Game}
\label{sec:general}

This section introduces our game model. It starts by identifying and justifying costs of both parties, namely those of the adversary (cf. Sec.~\ref{sub:costs_of_adversary}) and those of the defender (cf. Sec.~\ref{sub:costs_of_defender}). Once these costs are identified, we use them to formulate the game characterizing the interplay between the parties. As is usual in economic studies, we partition the utility into \textit{initial costs} ($\InitialCosts$) that are mandatory and independent of the use of attack (or defence), \textit{ongoing costs} ($\OngoingCosts$), which scale with how many times the attack (or defence) is used, and \textit{total reward}, $\textrm{TR},$ which is the reward and which is also assumed to be proportional to the number of usage. Hence, the utility of each actor is given by 
\begin{equation}
\label{eq:UtilityDEFGeneral}
\Utility^k = -\InitialCosts^k - \OngoingCosts^k + \textrm{TR}^k,
\end{equation}
where $k \in \{\ADV, \DEF\}.$ As expected (and used later below), ongoing costs $\OngoingCosts$ and total reward $\textrm{TR}$ can be represented by a single term for both adversary and defender. We started by denoting them separately in order to clearly communicate their origins and factors. 
    
\subsection{Adversary}    
\label{sub:costs_of_adversary}
The adversary's \textbf{action} set is defined by the attack methods she can use, her perturbation budget, and the fraction of samples she can attack. Each action has associated properties with regards to its cost structure and the attack success rate, depending on the model to be attacked (which is in turn the defender's chosen action).
    
\textbf{Initial costs} of an adversary, $\InitialCosts^\ADV$, may include gathering intelligence about the victim, stealing the targeted model, acquisition of suitable data to train a surrogate model, computational costs to train a surrogate model, hardware or software, and human capital. Notice that some  costs occur even though no attack is carried out. For example, as soon as an adversary contemplates to attack a model, time is spent to investigate the options, get information on the target, or prototype an attack.

\textbf{Ongoing costs} of an adversary, $\OngoingCosts^\ADV$,  are mainly characterized by the computational costs of the attack to calculate perturbations. They depend on the attack method(s), but other factors might contribute such as a fee charged by the model provider per prediction, if the adversary does not own the model (or has trained a surrogate model).
    
\textbf{Total rewards} of an adversary, $\textrm{TR}^\ADV$, are the rewards the adversary obtains for attacking the defender's classifier. They can be positive when the target is successfully deceived, or negative for failed attacks. While often close to zero, negative rewards are conceivable, if every failed attack allows the defender to learn something about the adversary's strategies which helps them to detect the adversary more successfully.

\subsection{Defender}
\label{sub:costs_of_defender}
The defender has many \textbf{actions} to choose from, where each action is a combination of choices with regard to the architecture of the model, the data used to train the model, the training algorithm itself which can already implement defence mechanisms such as adversarial training~\cite{madry2017towards}, and the inference mode~\cite{cohen2019certified}. Each unique combination of the above leads to a model, which is considered to be a pure strategy. Hence, in the context of the defender, the terms model and action are used interchangeably. Each of them has associated properties with regard to its cost structure, its accuracy on clean inputs, and the robust accuracy on adversarial examples crafted with a given attack method and strength, where the last two are adversary's actions. 

The \textbf{initial costs} of the defender, $\InitialCosts^\DEF,$ include, but are not limited to, gathering and labeling training data, training the model (computational costs) and human capital, e.g. hiring an expert to instantiate the classification pipeline. Therefore some of them, such as the data acquisition, are mostly independent of the defender's strategy, while others like the number of trained models or the complexity of the training procedure clearly depend on them.
  
\textbf{Ongoing costs} of the defender, $\OngoingCosts^\DEF,$ occur constantly per classified input. They mainly correspond to the computational costs for inference of a specific model. Ongoing costs for models in the same complexity class can be considered equivalent, but some defences~\cite{salman2019provably} recommend classifying each sample many times (subjected to some randomness), which increases ongoing costs by orders of magnitude.
   
\textbf{Total reward} of the defender. A rational defender will only train and deploy the model when she can draw positive reward, $\textrm{TR}^\DEF\!,$ (not necessarily monetary) from its operation. As mentioned above, we assume these rewards to be described on a per-sample basis, specifically for correct outputs only. Further, it is reasonable to expect a negative reward, i.e. some penalty term, when the model's output differs from the ground truth. The extent of this penalty might differ from sample to sample and might be smaller for benign samples, but larger for samples manipulated by the adversary. 

Following the above considerations, we model the game as a non-zero-sum game as both actions are clearly interdependent and one actor's positive reward is not necessarily equal to the other actor's negative reward. We further choose a simultaneous move game in which both actors decide on their strategy at the same time without having certainty about the opponent's chosen strategy.

\subsection{Cost of pure strategies}
\label{sec:formalization}
The defender classifies a finite number of samples $n\in\mathbb{N}$, out of which a fraction $\rmax\in[0,1]$ is under the control of the adversary, which means she can attack them by any method of her choice.\footnote{Note that this is already a simplification, since in practice none of the parties knows how many samples the adversary can influence.} The adversary may use up to $M-1$ different attacks (we refer to the $j$-th attack as attack $j$) and may also choose to leave a sample untouched (not to attack), which is denoted as attack $M$. 
The defender may choose one of $N$ different models to classify samples (we refer to the $i$-th model as model $i$). The accuracy of the $i$-th model is given by $\acc_{i}$ and its robustness against the $j$-th attack is given by $\rob_{ij}$. 

For expressing the costs, it will be useful to introduce two metrics: Attack success rate (ASR) and Correct classification rate (CCR). The first quantifies how successful the $j$-th attack is against the $i$-th classifier, while the second quantifies the probability of correctly classifying samples with model $i$, taking into account that samples might be attacked by attack $j$ with probability $\rho.$ Though both quantities seem to be similar at first glance, we have shown below that they allow to compactly represent the strategies and play a pivotal role in the analysis of the game in Sec.~\ref{sec:analysis}.
\begin{definition}[Attack Success Rate]\label{def:ASRij}
We define the attack success rate $\ASR$ of attack $j$ against model $i$ as 
\begin{equation}\label{eq:ASRij}
	\ASR_{ij} = 1 - \rob_{ij}.
\end{equation}
\end{definition}

\begin{definition}[Correct Classification Rate]\label{def:CCRij}
We define the correct classification $\CCR$ rate of model $i$, where only a fraction $\rho\in[0,\rmax]$ of all samples are perturbed by the $j$-th adversarial attack.
\begin{equation}\label{eq:CCRij}
	\CCR_{ij}(\rho) =  (1-\rho)\acc_i + \rho\ \rob_{ij}
\end{equation}
\end{definition}
Figure \ref{fig:Madry} shows the correct classification rate for for the ``wide'' architecture in \cite{madry2017towards} with $\rmax=1$.
The lines of the two models intersect at $\approx 17\%$, clearly indicating that the defender should use the undefended model if the proportion of adversarial examples is below, and the defended model when the proportion is above.


\subsubsection{Adversary}
Below, we state and discuss our assumptions on the economic factors of the adversary which are based on adversarial machine learning literature. 
\begin{axioms}{A}
    \item \label{item:A1} The initial costs $\InitialCosts^\ADV\ge0$ are constant and non-negative. This corresponds to a case where all attacks are similar in nature. That is, all attacks require a similar level of intelligence on the victim's model, use the same surrogate model (if applicable) and demand the same hardware, software and development costs. Moreover, this assumes that the adversary has to instantiate each attack to assess its quality. This incurs the initial costs, even if the adversary does not create a single adversarial example.
    \item \label{item:A2}Generating an adversarial example with attack $j$, where $j<M$, inflicts constant ongoing costs $\OngoingCosts_j^\ADV\ge0$. 
    \item \label{item:A3}Successful attacks yield revenue $R_+^\ADV\ge0$, while unsuccessful attacks cost $R_-^\ADV\ge0$. 
    \item \label{item:A4}Attack $j<M$ against model $i$ succeeds with probability $\ASR_{ij} = 1 - \rob_{ij}$ and fails with probability $1 - \ASR_{ij} = \rob_{ij}$. 
    \item \label{item:A5}Attack $M$ corresponds to refraining from attacking. No adversarial examples are generated, ongoing costs $\OngoingCosts_M^\ADV = 0$ are assumed to be equal to zero, though we assume an adversary has to pay the initial costs. Furthermore, $\ASR_{iM}$ is undefined and $\CCR_{iM} \equiv \acc_i$ for all models $i$.
    \end{axioms}
With the above assumptions, we define the expected payoff (revenue) per sample (EPPS) of the adversary when using the $j$-th attack against the $i$-th model as: 
\begin{align}
    \EPPS_{ij}^\ADV 
    &= -\OngoingCosts_j^\ADV - R_-^\ADV(1-\ASR_{ij}) + R_+^\ADV\ASR_{ij}
    \enspace,
\end{align}
where the minus in front of $R_-^\ADV(1-\ASR_{ij})$ indicates that adversary has to pay for failed attacks (or cannot get revenue, if $R_-^\ADV$ is equal to zero). The total revenue (utility) of the adversary after perturbing $n\rmax$ samples is 
\begin{equation}\label{eq:UtilityADV}
    U_{ij}^\ADV = 
        \begin{cases} 
            -\InitialCosts^\ADV + n\rmax\EPPS_{ij}^\ADV & \text{if $j<M$}\\
            -\InitialCosts^\ADV & \text{if $j=M$}
        \end{cases}\enspace,
\end{equation}
where we have used assumption~\eqref{item:A1} (initial costs are not influenced by strategy of any party) and~\eqref{item:A5} (refraining from attacking does not cost anything).

\subsubsection{Defender} 
Next, we state and discuss our assumptions on the economic factors of the defender, again based on the literature on adversarial machine learning.
\begin{axioms}{D}
    \item \label{item:D1} The defender has to classify all $n\in\mathbb{N}$ samples. This models a situation, where the defender cannot distinguish between any sample a priori and therefore does not exclude any of them. Note that adversarial examples are indistinguishable from benign examples by definition.   
    \item \label{item:D2} The initial costs $\InitialCosts^\DEF\ge0$ are constant and do not depend on the defender's choice of model. This covers all cases, where the costs for training each model are either the same or are dominated by the costs for setting up the training pipeline. 
    \item \label{item:D3} Classifying a sample with the $i$-th model inflicts ongoing costs $\OngoingCosts_i^\DEF>0$, which only depend on the model itself. 
    \item \label{item:D4} Defender rewards only depend on the correct classification rate. That is, correctly classifying a sample yields $R_+^\DEF\ge0$, while misclassifications incurs cost $R_-^\DEF\ge0$. In particular, errors on benign samples are as expensive as errors on adversarial examples.
\end{axioms}
Based on these assumptions we construct the utilities of the defender for the $i$-th model and the $j$-th attack. Remember that the adversary perturbs at most a fraction $\rmax\in[0,1]$ of all samples. By choosing any attack $j<M$ the adversary generates $n\rmax$ adversarial examples. The expected payoff (revenue) per classified sample (EPPS) for the defender is given by 
\begin{equation}\label{eq:EPPSij}
    \EPPS_{ij}^\DEF = 
        \begin{cases} 
            -\OngoingCosts_{i}^\DEF\! - R_-^\DEF\!(1-\CCR_{ij}(\rmax)) + R_+^\DEF\CCR_{ij}(\rmax) & \text{if $j<M$}\\
            -\OngoingCosts_{i}^\DEF\! - R_-^\DEF\!(1-\acc_i) + R_+^\DEF\acc_i & \text{if $j=M$}
        \end{cases}\enspace,
\end{equation}
where we have used assumption \eqref{item:D3} (classifying a sample incurs constant costs) and~\eqref{item:D4} (reward depends only on correct classification rate). By assumption~\eqref{item:A5} no adversarial examples are generated if $j=M$. By further incorporating assumption~\eqref{item:D1} (defender has to classify all samples) and \eqref{item:D2} (initial costs are constant), we get the defender's utility 
\begin{equation}\label{eq:UtilityDEF}
    \mathbf{U}_{ij}^\DEF = -\InitialCosts^\DEF + n \ \EPPS_{ij}^\DEF \enspace.
\end{equation}\noindent
\begin{table}[t]
	\centering
	\caption{Overview of Game parameters}
	\resizebox{\textwidth}{!}{
	\begin{tabular}{lll@{~~~}ll}
		\toprule
		\multirow{5}{*}{Description} & 
		\multicolumn{2}{c}{Total number of samples} & \multicolumn{2}{c}{Fraction of samples that can be} \\
		& 
		\multicolumn{2}{c}{to classify $n \in \mathbb{N}$} & \multicolumn{2}{c}{adversarially perturbed $\rmax\in[0,1]$} \\[1ex]
		& \multicolumn{2}{c}{Accuracy values of all models} & \multicolumn{2}{c}{Robustness values of all models} \\
		& \multicolumn{2}{c}{$\acc\in[0,1]^N$} & \multicolumn{2}{c}{ $\rob\in[0,1]^{N\times (M-1)}$} \\
		\cmidrule(r){2-5}
		& \multicolumn{2}{c}{Defender's side} 				& \multicolumn{2}{c}{Adversary's side} \\ 
		\cmidrule(r){2-3}\cmidrule(r){4-5}
		Positive reward 	& correct classification &  $R_+^\DEF\in[0,\infty)$ & successful attack & $R_+^\ADV\in[0,\infty)$ \\
		Negative reward	 	& misclassification      &  $R_-^\DEF\in[0,\infty)$ & failed attack     & $R_-^\ADV\in[0,\infty)$ \\
		Initial Costs       & 						 &  ~$\InitialCosts^\DEF\in[0,\infty)$ &			& ~$\InitialCosts^\ADV\in[0,\infty)$ \\
		Ongoing costs 		&						 &  $\OngoingCosts^\DEF\in[0,\infty)^N$&		& $\OngoingCosts^\ADV\in[0,\infty)^{M-1}$ \\
		\cmidrule(r){2-3}\cmidrule(r){4-5}
		Choice Parameter    &                  		 &  $s\in\StrategySetDEF$	& 					& $r\in\StrategySetADV$ \\ 
		Performance  & \multirow{2}{*}{correct classification rate} & \multirow{2}{*}{$\CCR(r)^\dagger$}	& \multirow{2}{*}{attack success rate}& \multirow{2}{*}{$\ASR(s)^\dagger$} \\
		Measure		 &&&&\\
		Expected Payoff & \multirow{2}{*}{classified} &\multirow{2}{*}{$\EPPS^\DEF\!(r)^\dagger$}		& \multirow{2}{*}{adversarially perturbed} &\multirow{2}{*}{$\EPPS^\ADV\!(s)^\dagger$} \\
		per Sample &&&&\\
		\bottomrule
		\multicolumn{3}{l}{\scriptsize $~^\dagger:$ Note that these entries depend on the opponent's choice}
	\end{tabular}
	}
	\label{tab:GameInstantiation}
\end{table}\noindent
\subsection{Utility of mixed strategies}
Let $\StrategySetADV = \Delta^M$ and $ \StrategySetDEF = \Delta^N$ denote the strategy space\footnote{$\Delta^d = \left\{v\in[0,1]^d \ \colon v_1 + \dots + v_d = 1 \right \}$ is the $d-1$ dimensional probability simplex.} of the adversary and defender, respectively. Each strategy $r \in \StrategySetADV$ / $s \in \StrategySetDEF$ corresponds to a probability distribution, where the adversary / defender chooses action $j$ / $i$ with probability $r_j$ / $s_i$, respectively.

The adversary's and defender's utility function are given by
\begin{align*}
    \Utility^\ADV\!(s,r) &= \sum_{i=1}^N\sum_{j=1}^M s_i \mathbf{U}_{ij}^\ADV r_j = s^T \mathbf{U}^\ADV r \\
    \Utility^\DEF\!(s,r) &= \sum_{i=1}^N\sum_{j=1}^M s_i \mathbf{U}_{ij}^\DEF r_j = s^T \mathbf{U}^\DEF r \enspace,
\end{align*}
respectively. For the analysis of the game, it is convenient to represent the utility functions as  
\begin{align}
    \Utility^\ADV\!(s,r) &= -\InitialCosts^\ADV + n\ \rmax \ r^T \ \EPPS^\ADV\!(s)\label{eq:UtilityADV_with_EPPS} \\
    \Utility^\DEF\!(s,r) &= -\InitialCosts^\DEF + n\ s^T \ \EPPS^\DEF\!(r)\label{eq:UtilityDEF_with_EPPS}  \enspace,
\end{align}
with the expected payoff per classified sample depending only their respective opponent's strategy $\EPPS^\ADV\!(s)\defeq(\EPPS^\ADV\!\ s)^T$ and $\EPPS^\DEF\!(r)\defeq\EPPS^\DEF\!\ r$. 


By this, the advanced adversarial classification game is fully described as a non-zero-sum normal form game with the adversary and defender as players, pure strategy sets $\{e_1^\DEF,\dots, e_N^\DEF\}$ and $\{e_1^\ADV,\dots, e_M^\ADV\}$, and utility functions $\Utility^\ADV\!$ and $\Utility^\DEF\!$, where $e_i^\DEF\in\mathbb{R}^N$ and $e_j^\ADV\in\mathbb{R}^M$ are the $i$-th and $j$-th standard unit vectors, respectively. The mixed strategy sets are given by $\StrategySetADV$ and $\StrategySetDEF$. All parameters of the game are shown in  Table~\ref{tab:GameInstantiation}. 


\subsection{Expected payouts for mixed strategies}
In this subsection we introduce two performance measures for both players which depend only on non-choice parameters and their respective opponent's strategy. These performance measures are generalizations of the attack success rate (\ASR) and correct classification rate (\CCR) in Definitions~\ref{def:ASRij} and \ref{def:CCRij}.
\begin{definition}[ASR for mixed strategies]
Given a defender's strategy $s\in\StrategySetDEF$, we define the attack success rate of attack $j$ as
\[\ASR_j(s) \defeq s^T (\vec{1} -\rob_{\cdot j})  = 1-\sum_{i=1}^N s_i\rob_{ij} \enspace. \]
\end{definition}

This is a direct extension of the attack success rate from  Definition~\ref{def:ASRij} to mixed strategies, since $\ASR_j(e_i^\DEF)=\ASR_{ij}$. As the name implies, it quantifies the probability of a successful attack given a possibly mixed defender's strategy.


\begin{definition}[CCR for mixed strategies]\label{def:CCR}
We define the correct classification rate with respect to an adversary's strategy $r\in\StrategySetADV$ against model $i$ as
\[\CCR_{i}(r;\rmax) \defeq r^T  \left(\CCR_{ij}(\rmax)\right)_{1\le j\le M} = \sum_{j=1}^M r_j\CCR_{ij}(\rmax)\enspace,
\]
where $\CCR_{iM}\equiv \acc_i$ by assumption \eqref{item:A5}. 
\end{definition}

Similar to above, this generalizes the correct classification rate from Definition~\ref{def:CCRij} to mixed strategies. Unlike the above, we place an additional restriction on the performance metric. Here, the proportion of adversarial examples is explicitly restricted to $\rmax\in[0,1]$. For pure strategies of the adversary we get $\CCR_i(e_j^\ADV;\rmax)=\CCR_{ij}(\rmax)$, the probability of correctly classifying a sample if the adversary chooses attack $j$.

\subsubsection{Expected payoff per sample} 
Both performance measures allow us to concisely describe each entry of the $\EPPS$-functions in Equation~\eqref{eq:UtilityADV_with_EPPS} and \eqref{eq:UtilityDEF_with_EPPS} 
\begin{align} \nonumber
    \EPPS_j^\ADV\!(s) &= -\OngoingCosts_j^\ADV - R_-^\ADV(1-\ASR_j(s)) + R_+^\ADV\ASR_j(s)  \\
    &= -\OngoingCosts_j^\ADV - R_-^\ADV - (R_+^\ADV+R_-^\ADV)\ \ASR_j(s) \label{eq:EPPSadv}\\ \nonumber
    \EPPS_i^\DEF\!(r) &= -\OngoingCosts_i^\DEF - R_-^\DEF(1-\CCR_i(r;\rmax)) + R_+^\DEF\CCR_i(r;\rmax) \\
    &= -\OngoingCosts_i^\DEF - R_-^\DEF + (R_+^\DEF+R_-^\DEF)\ \CCR_i(r;\rmax)\enspace. \label{eq:EPPSdef}
\end{align} 
$\EPPS_j^\ADV\!(s)$ corresponds to the expected payoff per adversarially perturbed sample for attack $j$ given the strategy $s\in\StrategySetDEF$. Similarly, $\EPPS_i^\DEF\!(r)$ corresponds to the expected payoff per classified sample given model $i$ and strategy $r\in\StrategySetADV$.


\section{Game Instantiation and Analysis}
\label{sec:analysis}
Below, we analyse the simplest realization of the game, where $N=2$ and $M=2$. The defender chooses from two classifiers (typically one of them trained as usual and the other using some form of adversarial training increasing robustness) and the adversary may attack or not. Hence their strategy spaces can be specified by a single scalar parameter $r = (r_1, 1-r_1)$ and $s = (s_1, 1-s_1)$ with $r_1, s_1\in[0,1]$ for the adversary and defender, respectively. By this, ``not attacking'', i.e., attack $M=2$ is represented by the second entry ($1-r_1$). We further simplify the notation by writing $\rob_{i}$ instead of $\rob_{i1}$ and $\ASR(s)$ instead of $\ASR_1(s)$.

As a convention, motivated by experimental results of state of the art methods for adversarial training, we assume the first model (trained normally) to have higher accuracy on benign samples but lower on adversarial samples (robustness) compared to the second model. Furthermore, accuracy on attacked samples of both models is assumed to be lower than that on clean samples, since the adversary aims to cause a misclassification and not to help the defender. This means the accuracy and robustness are ordered as 
\begin{equation}
    \acc_1 > \acc_2 > \rob_2 > \rob_1 \enspace.
    \label{eq:ordering}
\end{equation}
This order also makes sense from a game-theoretical point of view. If the first model would have both, higher accuracy and higher robustness, choosing this model would be a strictly dominant strategy and the defender will never use the second model, assuming equal ongoing costs as has been justified above. Contrary, if the first inequality would not hold, choosing model 2 would be strictly dominant. 
In order to simplify the equations in this section, we define
\begin{align*}
    \Dacc &\defeq \acc_1 - \acc_2 > 0\\
    \Drob &\defeq \rob_2 - \rob_1 > 0
\end{align*}

\subsection{Best response analysis of the adversary}\label{sec:BRAadv}
Assuming a fixed strategy of the defender $s\in\StrategySetDEF$, a best response of the adversary maximizes the utility $r^* \in \arg\max_{r\in\StrategySetADV}\Utility^\ADV\!(s,r)$ defined in Equation~\eqref{eq:UtilityADV_with_EPPS}. Since the initial costs $\InitialCosts^\ADV$ are constant (by assumption~\eqref{item:A1}), the utility linearly increases in $r$, hence the maximization is trivial depending on the sign of her expected payout per sample, $\EPPS^\ADV\!(s),$ as follows:
\begin{equation}\label{eq:casesEPPS_ADV}
r_1^* \in \begin{cases}
\{0\} & \text{ iff } \EPPS^\ADV\!(s) < 0\\
\{1\} & \text{ iff } \EPPS^\ADV\!(s) > 0\\
[0,1] & \text{ iff } \EPPS^\ADV\!(s) = 0
\end{cases}
\enspace.
\end{equation}
The last line means that any $r^*\in\StrategySetADV$ is a best response. We refer to these cases as Case 1 (never attack), Case 2 (always attack), and Case 3 (indifferent) in that order. The definition of $\EPPS^\ADV\!$ from Equation~\eqref{eq:EPPSadv} can be reformulated as
\begin{equation*}
\EPPS^\ADV\!(s) = \left(R_+^\ADV + R_-^\ADV\right)(\ASR(s) - \mu^\ADV)\enspace,
\end{equation*}
where
\begin{equation*}
\mu^\ADV \defeq \frac{\OngoingCosts^\ADV + R_-^\ADV}{R_+^\ADV + R_-^\ADV}\enspace,
\end{equation*}
which leads to an alternative characterization of the best responses from Equation~\eqref{eq:casesEPPS_ADV}
\begin{equation}\label{eq:casesASR}
r_1^* \in \begin{cases}
\{0\} & \text{ iff }  \ASR(s) < \mu^\ADV\\
\{1\} & \text{ iff }  \ASR(s) > \mu^\ADV\\
[0,1] & \text{ iff }  \ASR(s) = \mu^\ADV
\end{cases} \enspace.
\end{equation}
The last Equation~\eqref{eq:casesASR} relates the attack success rate $\ASR$ to the economic factors of adversary rewards, as $\mu^\ADV$ is defined in terms of her rewards and ongoing costs. 
Assuming the penalty for a failed attack $R^\ADV_-$ is negligible (presently, most crimes of this type are left unpunished due to lack of legislation, law enforcement, and forensic tools) and the reward $R^\ADV_+$ dominates the ongoing costs, $\mu^\ADV$ is in practice going to be close to zero.
In consequence this means that if there is a slight chance of an attack to succeed, every rational adversary will always attack. 


\begin{table}[t]
	\centering
	\caption{All possible best responses $r^*=(r_1^*,1-r_1^*)\in\StrategySetADV$ of the adversary depending on the defender's strategy $s=(s_1,1-s_1)\in\StrategySetDEF$. The first, second and third row summarize Case 1, 2, and 3, respectively. If the precondition does not hold, then the condition on $s\in\StrategySetDEF$ cannot be satisfied.}
	\begin{tabular}{@{~~~}c@{~~~}l@{~~~~~~~}c@{~~~~~}r}
		\toprule
		Case & Best responses  &                     Condition                     &                                         Precondition \\ \midrule
                     1               & $r^*=(0,1)$           & $\ASR(s) < \mu^\ADV$ & $\hphantom{\rob_1<{}}1-\mu^\ADV<\rob_2$ \\
                     2               & $r^*=(1,0)$       & $\ASR(s) > \mu^\ADV$ & $\rob_1<1-\mu^\ADV\hphantom{{}<\rob_2}$ \\
                     3               & $r^*\in\StrategySetADV$ & $\ASR(s) = \mu^\ADV$ &           $\rob_1\le1-\mu^\ADV\le\rob_2$ \\ \bottomrule
	\end{tabular}
	\label{tab:AdversaryBestResponse}
\end{table}
We summarize our results in Table~\ref{tab:AdversaryBestResponse}. We specify the case on the left-most column. The best responses for the given case are shown in the second column. For each case, we have a corresponding equivalent condition on the $\ASR(s)$, which is shown in the third column. The fourth and last column shows for each case if it is satisfiable at all by any $s\in\StrategySetDEF$. These preconditions depend only on non-choice parameters in Table~\ref{tab:GameInstantiation} 
and not on the defender's strategy itself.

\subsection{Best response analysis of the defender}\label{sec:BRAdef}
For a given adversary strategy $r\in\StrategySetADV$, a best response of the defender maximizes the utility $s^*\in\arg\max_{s\in\StrategySetDEF}\Utility^\DEF\!(s,r)$, see Equation~\eqref{eq:UtilityDEF_with_EPPS}. This is equivalent to maximizing the expected payout per sample $s^T \EPPS^\DEF\!(r)$ as defined in Equation~\eqref{eq:EPPSdef}, which we rewrite as 
\begin{equation}\label{eq:EPPSdef_frac}
    \EPPS_i^\DEF(r) = \left(R_+^\DEF+R_-^\DEF\right)\left(\CCR_i(r;\rmax) - \mu_i^\DEF\right)\enspace,
\end{equation}
where 
\[\CCR_i(r;\rmax) = (1-r_1 \rmax)\acc_i + r_1\rmax\ \rob_i\]
and
\[\mu_i^\DEF = \frac{\OngoingCosts_i^\DEF + R_-^\DEF}{R_+^\DEF + R_-^\DEF}\enspace.\] 

We observe that similarly to the adversary's case, the defender chooses a model $i$ with maximal $\CCR_i(r;\rmax) - \mu_i^\DEF$. Here again the performance metric $\CCR_i(r;\rmax)$ is related to the economic setting term $\mu_i^\DEF$. Importantly and unlike the adversary's case, the choice of the defender's strategy depends on $\rmax$, the maximal fraction of samples the adversary can influence. As will be seen below, this means that for some small $\rmax$ (which occurs in many practical situations), the defender will not have an incentive to use the robust classifier.\\

For the analysis of Nash equilibria in the next section, it is useful to analyze the difference $\EPPS_1^\DEF - \EPPS_2^\DEF$, 
\begin{equation}\label{eq:EPPSdiff}
\DCCR - \Delta\mu^\DEF = \frac{\EPPS_1^\DEF\!(r)-\EPPS_2^\DEF\!(r)}{R_+^\DEF+R_-^\DEF}
\enspace,
\end{equation}
where
\begin{align*}
\DCCR &\defeq \CCR_1(r;\rmax)-\CCR_2(r;\rmax),\text{ and}\\
\Delta\mu^\DEF &\defeq \mu_1^\DEF - \mu_2^\DEF =  \frac{\OngoingCosts_1^\DEF-\OngoingCosts_2^\DEF}{R_+^\DEF + R_-^\DEF}\enspace.
\end{align*}
Notice that in many practical situations, the difference in ongoing costs of two classifiers is almost zero (non-robust and robust versions differ mainly in training, not in the architecture of the model, which makes inference costs akin), therefore $\Delta\mu^\DEF\approx 0$. 
The best responses for the defender are given by
\begin{equation}\label{eq:casesCCR}
s_1^* \in \begin{cases}
\{0\} & \text{ iff }  \DCCR < \Delta\mu^\DEF\\
\{1\} & \text{ iff }  \DCCR > \Delta\mu^\DEF\\
[0,1] & \text{ iff }  \DCCR = \Delta\mu^\DEF
\end{cases} \enspace.
\end{equation}
The last line means that any $s^*\in\StrategySetDEF$ is a best response. We refer to these cases as Case A (always defend), Case B (never defend), and Case C (indifferent) in that order.



Defender's best responses are summarized in Table~\ref{tab:DefenderBestResponse}. The case is specified in the left-most column, the best responses for the given case are shown in the second column, and the corresponding equivalent condition on $r$ is shown in the third column. The fourth and last column shows prerequisites for the given case that only depend on non-choice parameters.

\begin{table}[t]
	\centering
	\caption{All possible best responses $s^*=(s_1^*,1-s_1^*)\in\StrategySetDEF$ of the defender depending on the adversary's strategy $r=(r_1,1-r_1)\in\StrategySetADV$. The first, second and third row summarize Case A, B, and C, respectively. If the precondition does not hold true, then the condition on $r\in\StrategySetADV$ cannot be satisfied.}
	\begin{tabular}{@{~~~}c@{~~~}l@{~~~~~~~}cr}
		\toprule
		Case & Best responses &   Condition    &                                  Precondition \\ \midrule
		             A               & $s^*=(0,1)$      & $\DCCR < \Delta\mu^\DEF$ & \phantom{as}$\phantom{0<{}}\precondDEF<\rmax$ \\
		             B               & $s^*=(1,0)$      & $\DCCR > \Delta\mu^\DEF$ & \phantom{as}$\Delta\mu^\DEF<\Dacc$\\
		             C               & $s^*\in\StrategySetDEF$ & $\DCCR = \Delta\mu^\DEF$ & \phantom{as}$0\le\precondDEF\le\rmax$ \\ \bottomrule
	\end{tabular}
	\label{tab:DefenderBestResponse}
\end{table}

\subsection{(Fully) Mixed Nash equilibria} \label{sec:mixednash}
In Subsection \ref{sec:BRAadv} and \ref{sec:BRAdef} we discussed and listed all possible best responses for both adversary and defender. The results are summarized in Table \ref{tab:AdversaryBestResponse} and \ref{tab:DefenderBestResponse}. Now we investigate if and when mixed strategy Nash equilibria exist at all. For this we consider Case 3 and Case C, that is
\[\ASR(s) = \mu^\ADV ~~\text{and}~~ \DCCR = \Delta\mu^\DEF\]
or, equivalently,
\begin{equation}\label{eq:explicit_s1_r1}
    s_1 = \condADV ~~\text{and}~~ r_1\rmax = \condDEF,
\end{equation}
respectively. Obviously, $s\in\StrategySetDEF$ and $r\in\StrategySetADV$ are mixed strategies, and therefore  $s_1,r_1\in(0,1)$, if and only if 
\begin{equation}\label{eq:fullymixedcond}
    \rob_1<1-\mu^\ADV<\rob_2 ~~\text{and}~~0<\condDEF<\rmax\enspace,
\end{equation} 
respectively.

\newtheorem{thm}{Theorem}
\begin{thm}
Let $\rob_1<1-\mu^\ADV<\rob_2$ and $0<\precondDEF<\rmax$. Then the fully mixed strategy Nash equilibrium $(s^*,r^*)$ given by
\begin{equation}\label{eq:mixednash}
s_1^* = \condADV ~~\text{and}~~ r_1^*\rmax= \condDEF\enspace.
\end{equation}
is unique.
\end{thm}
\begin{proof}
Consider Tables~\ref{tab:AdversaryBestResponse} and \ref{tab:DefenderBestResponse}. By the analysis above, $s_1^*$ and $r_1^*$ are given as in Equation~\eqref{eq:explicit_s1_r1} and therefore the conditions for Case 3 and Case~C are satisfied. That is all $r\in\StrategySetADV$ are a best response to $s^*$ and all $s\in\StrategySetDEF$ are a best response to $r^*$. Furthermore, the conditions in Equation~\eqref{eq:fullymixedcond} are fulfilled and therefore both strategies are mixed strategies. In conclusion, $(s^*,r^*)$ is a fully mixed strategy Nash equilibrium.\\
In order to prove uniqueness of the Nash equilibrium, we first assume there exists another Nash equilibrium $(\hat{s},\hat{r})$. That is, either $\hat{s}_1\neq s_1^*$ or $\hat{r}_1\neq r_1^*$.  We consider only the case $\hat{s}<s^*$, as all other cases are conducted analogously. Observe that 
\[s\mapsto \ASR(s)=1-\rob_2+s_1\Drob\]
strictly increases in $s_1$ and
\[r\mapsto \DCCR=\Dacc - r_1\rmax(\Dacc + \Drob)\]
strictly decreases in $r_1$. Since $\hat{s}$ and $\hat{r}$ are mutual best responses, we have
\begin{align*}
\hat{s}_1<s_1^*&\implies \ASR(\hat{s})<\ASR(s^*)=\mu^\ADV\\
            &\implies \hat{r}=(0,1)\\
            &\implies \Delta\CCR(\hat{r};\rmax) > \Delta\CCR(r^*;\rmax) = \mu^\ADV\\
            &\implies \hat{s}=(1,0) \text{\vspace{1cm}} \implies \hat{s}_1>s_1^* \enspace,
\end{align*}
which is a contradiction. All in all, $(s^*,r^*)$ is a unique Nash equilibrium. \qed

\end{proof}
\subsection{Results}


\begin{figure}[t]
	\centering
\begin{tikzpicture}[scale=1]

\definecolor{color0}{rgb}{0.75,0.75,0}
\definecolor{color1}{rgb}{0,0.75,0.75}

\begin{axis}[
legend cell align={left},
legend style={
  fill opacity=0.8, 
    draw opacity=1, 
    text opacity=1, 
    draw=black,
  at={(1.05,0.02)},
  anchor=south west,
  draw=white!80!black,
  font=\tiny
},
tick align=outside,
tick pos=left,
x grid style={white!69.0196078431373!black},
xlabel={\(\displaystyle \rob_2\)},
xmin=0, xmax=1,
xtick style={color=black},
xtick={0,0.2,0.4,0.6,1},
xticklabels={0.0,0.2,0.4,0.6,1.0},
y grid style={white!69.0196078431373!black},
ylabel={\(\displaystyle \rob_1\)},
ymin=0, ymax=1,
ytick style={color=black},
unit vector ratio=1 .7
]


\path [draw=white!76.8627450980392!black, fill=white!76.8627450980392!black, opacity=0.145098039215686]
(axis cs:1,0.8)
--(axis cs:0.8,0.8)
--(axis cs:0.8,1)
--(axis cs:1,1)
--cycle;

\path [draw=white!41.9607843137255!black, fill=white!41.9607843137255!black, opacity=0.145098039215686]
(axis cs:0.8,0.8)
--(axis cs:0.8,0)
--(axis cs:0,0)
--(axis cs:0,0.8)
--cycle;

\path [draw=white!41.9607843137255!black, fill=black, opacity=1]
(axis cs:0,0.8)
--(axis cs:0,1)
--(axis cs:0.8,1)
--(axis cs:0.8,0.8)
--cycle;
\path [draw=white, postaction={pattern=north west lines, pattern color=white}]
(axis cs:0,0)
--(axis cs:1,1)
--(axis cs:.8,1)
--(axis cs:.8,.8)
--(axis cs:0,.8)
--cycle;
\addplot [semithick, gray, dotted, forget plot]
table {%
0 0.8
1 0.8
};
\addplot [semithick, gray, dotted, forget plot]
table {%
0.8 0
0.8 1
};

\addplot [semithick, only marks, color0, mark=*, mark size=3, mark options={solid}]
table {%
0.3392 0
};
\addlegendentry{Standard - \cite{shafahi2019adversarial}, m=2}
\addplot [semithick, only marks, blue, mark=*, mark size=3, mark options={solid}]
table {%
0.4115 0
};
\addlegendentry{Standard - \cite{shafahi2019adversarial}, m=4}
\addplot [semithick, only marks, green!50!black, mark=*, mark size=3, mark options={solid}]
table {%
0.4682 0
};
\addlegendentry{Standard - \cite{shafahi2019adversarial}, m=8}

\addplot [semithick, only marks, color1, mark=*, mark size=3, mark options={solid}]
table {%
0.4631 0
};
\addlegendentry{Standard - \cite{shafahi2019adversarial}, m=10}

\addplot [semithick, only marks, red, mark=asterisk, mark size=3, mark options={solid}]
table {%
0.458 0.035
};
\addlegendentry{Standard - \cite{madry2017towards} AT}

\addplot [semithick, black, forget plot]
table {%
0 0
1 1
};
\draw (axis cs:0.35,0.15) node[
  scale=0.75,
  anchor=base west,
  text=black,
  rotate=0.0
]{Case 2};
\draw (axis cs:0.8725,0.825) node[
  scale=0.75,
  anchor=base west,
  text=black,
  rotate=0.0
]{Case 1};
\draw (axis cs:0.815,0.3) node[
  scale=0.75,
  anchor=base west,
  text=black,
  rotate=0.0,
  align=left
]{ Case 3 \\(and 1\&2)\\possible};
\draw (axis cs:0.25,0.825) node[
  scale=0.75,
  inner sep=0pt,
  anchor=south,
  text=white,
  rotate=0.0
]{\textcolor{white}{$\uparrow$ $\neg$ Case 2}};
\draw (axis cs:0.8,0.25) node[
  scale=.75,
  inner sep=0pt,
  anchor=south east,
  text=black,
  rotate=0.0
]{$\neg$ Case 1 $\leftarrow$};
\end{axis}


\draw[overlay, remember picture] (5.48,0) -- (5.48,-0.4) node[below] {$1-\mu^\ADV$};

\draw[overlay, remember picture] (6.8,4.56*0.59) -- (7.15,4.56*0.59) node[right] {$1-\mu^\ADV$};

\end{tikzpicture}
	\caption{Adversary's preferences with regards to $\rob_1$ and $\rob_2$ for $\mu^\ADV$. Decreasing $\mu^\ADV$ shifts the horizontal and vertical lines (captioned ``$1  -\mu^\ADV$'') up and right. This means that area of Case 2 (always attack) is getting bigger and the area Case 3 (mixed strategy is possible) is getting smaller. As discussed earlier, in practice $\mu^\ADV \approx 0.$ This would mean the adversary will always attack. Note that the area above the minor diagonal line is unreachable, because of the assumption in Eq.~\eqref{eq:ordering}. }
	\label{fig:adv_case}
\end{figure}
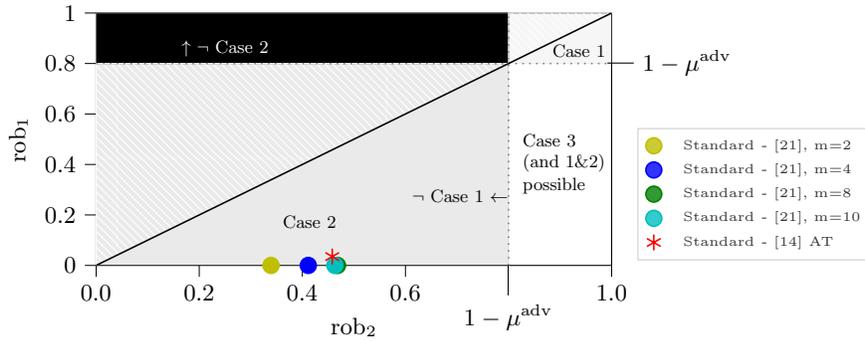

In the previous section, we have identified best responses and their preconditions for both actors given the opponent's strategy (see Tables~\ref{tab:AdversaryBestResponse} and~\ref{tab:DefenderBestResponse}). They allow to identify the set of strategies available to each rationally behaving actor for the given economic factors.


We visualize the adversary's options in Figure~\ref{fig:adv_case}, where we can see that $\mu^\ADV$ (see lines captioned ``$1 - \mu^\ADV$'') determines the areas, where the adversary will always attack (Case 2), never attack (Case 1), and where we cannot say without considering the defender's actions and she might play a mixed strategy (Case 3). Note that the area above minor diagonal is unreachable (by Eq.~\eqref{eq:ordering}). We can observe that most area is covered by Case 2, which means that the adversary is incentivized to always attack, especially if a value of $\mu^\ADV$ is low, which happens if the costs for being caught $R_-^\ADV$ and her ongoing cost $\OngoingCosts^\ADV$ are low and when the potential reward $R_+^\ADV$ is high. Note that the black area, where neither Case 1, nor Case 2 (and thus also not Case 3) is true, does never fulfill the ordering in Equation~\eqref{eq:ordering} and thus is undefined in our setting.



\begin{figure}[t]
	\centering
	\newcommand{\strength}{20}
\begin{tikzpicture}[scale=1]

\definecolor{color0}{rgb}{0.75,0.75,0}
\definecolor{color1}{rgb}{0,0.75,0.75}

\begin{axis}[
legend cell align={left},
legend style={
    fill opacity=0.8, 
    draw opacity=1, 
    text opacity=1, 
    draw=black,
    at={(1.05,0.35)},
    anchor=south west,
    font=\tiny
    },
tick align=outside,
tick pos=left,
x grid style={white!69.0196078431373!black},
xlabel={\(\displaystyle \Drob\)},
xmin=0, xmax=1,
xtick style={color=black},
xtick={0,0.2,0.4,0.6,0.8,1},
xticklabels={0.0,0.2,0.4,0.6,0.8,1.0},
y grid style={white!69.0196078431373!black},
ylabel={\(\displaystyle \Dacc\)},
ymin=-.3, ymax=1,
ytick style={color=black},
unit vector ratio=1 .7
]
\path [draw=white!76.8627450980392!black, fill=white!76.8627450980392!black, opacity=0.145098039215686]
(axis cs:0,-0)
--(axis cs:0.525,0.475)
--(axis cs:0,1)
--cycle;
\path [draw=white!41.9607843137255!black, fill=white!41.9607843137255!black, opacity=0.145098039215686]
(axis cs:0.,0.)
--(axis cs:0.,-.5)
--(axis cs:1,-.5)
--(axis cs:1,0)
--cycle;


\path [draw=white!41.9607843137255!black, draw opacity=0.145098039215686, postaction={pattern=north east lines, pattern color=white!41.9607843137255!black, fill opacity=0.145098039215686}]
(axis cs:0,1)
--(axis cs:1,1)
--(axis cs:1,0)
--cycle;
\addplot [semithick, gray, dotted, forget plot]
table {%
0 0
1 0
};

\addplot [black, forget plot]
table {%
0 0
1 0
};
\addplot [semithick, gray, dotted, forget plot]
table {%
0 0
1 0.9
};

\addplot [semithick, only marks, color0, mark=*, mark size=3, mark options={solid}]
table {%
0.3392 0.0356
};
\addlegendentry{Standard vs \cite{shafahi2019adversarial}, m=2}
\addplot [semithick, only marks, blue, mark=*, mark size=3, mark options={solid}]
table {%
0.4115 0.0718
};
\addlegendentry{Standard vs \cite{shafahi2019adversarial}, m=4}
\addplot [semithick, only marks, green!50!black, mark=*, mark size=3, mark options={solid}]
table {%
0.4682 0.0904999999999999
};
\addlegendentry{Standard vs \cite{shafahi2019adversarial}, m=8}
\addplot [semithick, only marks, color1, mark=*, mark size=3, mark options={solid}]
table {%
0.4631 0.1107
};
\addlegendentry{Standard vs \cite{shafahi2019adversarial}, m=10}
\addplot [semithick, black, forget plot]
table {%
0 1
1 0
};
\addplot [semithick, only marks, red, mark=asterisk, mark size=3, mark options={solid}]
table {%
0.423 0.079
};
\addlegendentry{Standard vs \cite{madry2017towards}}
\draw (axis cs:0.27,0.2) node[
  scale=0.75,
  anchor=base west,
  text=black,
  rotate=0.0,
  align=left
]{Case C (and A\&B) possible};
\draw (axis cs:0.15,0.45) node[
  scale=0.75,
  anchor=base west,
  text=black,
  rotate=0.0
]{Case B};
\draw (axis cs:0.43,-0.27) node[
  scale=0.75,
  anchor=base west,
  text=black,
  rotate=0.0
]{Case A};
\draw (axis cs:0.6,0.65) node[
  scale=0.75,
  inner sep =0pt,
  anchor=base east,
  text=black,
  rotate=0.0
]{$\neg$ Case A $\nwarrow$};
\draw (axis cs:0.465,-0.025) node[
  scale=0.75,
  inner sep=0pt,
  anchor=north,
  text=black,
  rotate=0.0
]{$\downarrow$ $\neg$ Case B};
\end{axis}


\draw[overlay, remember picture] (6.1,3.7) -- (7.15,3.7) node[right] {$\rmax = \frac{\Dacc}{\Dacc+\Drob}$};

\draw[overlay, remember picture] (6.85,1) -- (7.15,1.) node[right] {$\Delta\mu^\DEF$};

\end{tikzpicture}
	\caption{Visualization of the defender cases with regards to $\Dacc$ and $\Drob$ for $\Delta\mu^\DEF = 0$ and $\rmax = 0.45$. Note that lower (negative) $\Delta\mu^\DEF$ values shift the horizontal line that indicates the area where $\neg$ Case B downwards and the diagonal line that indicates $\neg$ Case A to the right.}
	\label{fig:def_case}
\end{figure}
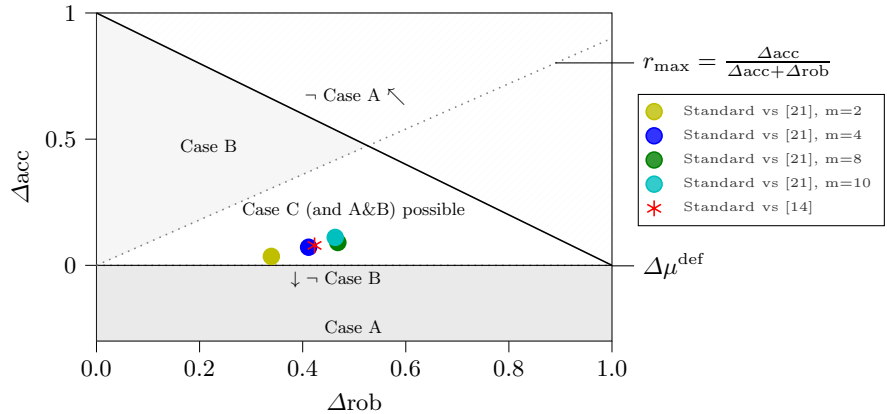

The same visualization for the defender is shown in Figure~\ref{fig:def_case} for a given $\Delta\mu^\DEF = 0$ and $\rmax = 0.45$.
Reachable areas are Case B (between the solid and dotted line), where the defender will never defend, and Case C where possibly a mixed strategy occurs. 
Note, while the slope of the solid line is fixed (by the ordering in Eq.~\eqref{eq:ordering}), the slope of the dotted line depends only on the value of $\rmax$ \footnote{An alternative formulation of the linear equation for the dotted line in Fig.~\ref{fig:def_case} is: $\Dacc = \frac{\rmax\Drob}{1-\rmax}$ (for $\Delta\mu^\DEF=0)$}. 
With decreasing value of $\rmax$, the slope will also decrease,
which means that the defender will have less incentive to use the robust classifier (the area of Case B will increase). 
Therefore if the proportion of samples the adversary can influence (or attack) $\rmax$ is small, a rational defender might not have an incentive to use the robust model, regardless of the strategy of the adversary. The case, where the defender will always use the robust classifier is not considered, as it would correspond to a case when the robust model has lower costs than the non-robust, even though it might have a lower $\CCR$ (but such a pure strategy can be still a solution of Case C).

The solid diagonal line depicts the condition that $\Dacc + \Drob < 1$, so all points to the right (the striped area) are not valid.\footnote{Note that $\Dacc+\Drob=\acc_1 - (\acc_2 - \rob_2) - \rob_1 < \acc_1 \le 1$, by Equation~\eqref{eq:ordering}.} Further, the ordering (Eq.~\eqref{eq:ordering}) ensures that $\acc_1 > \acc_2$ and thus $\Dacc > 0$. Therefore, all points below the horizontal line at $\Dacc = 0$ are invalid. We include this area into the figure to illustrate when the defender would purely deploy the second model (Case A) without considering the adversary's strategy. This is only the case when $\Delta\mu^\DEF$ is positive which in turn is only true if $\OngoingCosts_1^\DEF > \OngoingCosts_2^\DEF$, i.e. in our scenario, the defended model has lower ongoing costs than the undefended model. 
%


Colored dots and a star in Figure~\ref{fig:adv_case} and Figure~\ref{fig:def_case} illustrate the CIFAR-10 models proposed in \cite{madry2017towards} and \cite{shafahi2019adversarial}, where the accuracy and robustness values are taken from corresponding publications. We see in Figure~\ref{fig:adv_case} that for these values and the chosen $\mu^\ADV,$ an adversary would always attack, independent of the strategy the defender chooses (Case 1).  
Similarly in Figure~\ref{fig:def_case}, we can see that for the given values and the chosen $\mu^\DEF$ and $\rmax$, the defender is always in Case C, meaning that she might use both models.
Keep in mind that values of $\rmax, \mu^\ADV$ and $\mu^\DEF$ are arbitrarily chosen and the defender will play a model with higher $\CCR$ for that given attack rate (recall that the adversary will likely attack). Needless to say that our chosen value $\rmax = 0.45$ is very high as it means that the adversary can influence up to $45\%$ of samples. Realistically $\rmax$ will be much lower, even as low as one percent, which means that the defender might opt to use the non-robust model, although the adversary will always attack.

\section{Discussion}
\label{sec:discussion}

\begin{figure}[t]
	\centering
\begin{tikzpicture}[spy using outlines={circle, magnification=4, size=1.5cm, every spy on node/.append style={thick},connect spies}]

\definecolor{color0}{rgb}{0.12156862745098,0.466666666666667,0.705882352941177}
\definecolor{color1}{rgb}{1,0.498039215686275,0.0549019607843137}
\definecolor{color2}{rgb}{0.172549019607843,0.627450980392157,0.172549019607843}
\definecolor{color3}{rgb}{0.83921568627451,0.152941176470588,0.156862745098039}
\definecolor{color4}{rgb}{0.580392156862745,0.403921568627451,0.741176470588235}

\begin{axis}[
width=.75\textwidth,
height=.45\textwidth,
legend cell align={left},
legend style={fill opacity=0.8, draw opacity=1, text opacity=1, at={(0.03,0.03)}, anchor=south west, draw=white!80!black},
tick align=outside,
tick pos=left,
x grid style={white!69.0196078431373!black},
xlabel={Adversary's Strategy (attacking probability)},
xmajorgrids,
xmin=-0.05, xmax=1.05,
xtick style={color=black},
y grid style={white!69.0196078431373!black},
ylabel={Correct classification rate},
ymajorgrids,
ymin=-0.04, ymax=1.04,
ytick style={color=black}
]
\addplot [thick, color0]
table {%
0 0.9501
1 0.000
};
\addlegendentry{Standard}
\addplot [thick, color1]
table {%
0 0.9145
1 0.3392
};
\addlegendentry{m=2}
\addplot [thick, color2]
table {%
0 0.8783
1 0.4115
};
\addlegendentry{m=4}
\addplot [thick, color3]
table {%
0 0.8596
1 0.4682
};
\addlegendentry{m=8}
\addplot [thick, color4]
table {%
0 0.8394
1 0.4631
};
\addlegendentry{m=10}
\coordinate (SpyArea1) at (axis cs: 0.1,0.88);
\coordinate (SpyZoom1) at (axis cs: 0.7,0.8);

\coordinate (SpyArea2) at (axis cs: 0.3,0.75);
\coordinate (SpyZoom2) at (axis cs: 0.7,0.2);

\end{axis}
\spy on (SpyArea1) in node[fill=white] at (SpyZoom1);
\spy on (SpyArea2) in node[fill=white] at (SpyZoom2);
\end{tikzpicture}
	\caption{Correct classification rate on CIFAR-10 based on the 
	original data in 
	\cite{shafahi2019adversarial}.}
	\label{fig:Free}
\end{figure}
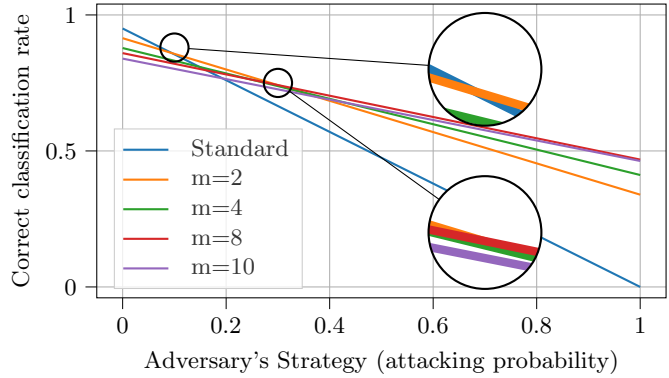

First of all, we extend the analysis of the $\CCR$, as already shown in Figure~\ref{fig:Madry} to the results reported in~\cite{shafahi2019adversarial} in Figure~\ref{fig:Free}. Here, we can see that of all the countermeasures proposed (for CIFAR-10 data), only $m=2$ and $m=8$ would be considered by the defender, again depending on the strategy of the adversary. If the adversary would choose to attack in less than $\approx 10\%$ of the cases or $\rmax \leq 0.1$, the undefended model would be strictly preferable. The proposed ideal solution from~\cite{shafahi2019adversarial}, $m=8$, is only optimal if the adversary chooses to attack with an probability of more than $\approx 30\%$ (or if $\rmax > 0.3$). In between, the countermeasure with $m=2$ is the defender's optimal strategy and all other strategies ($m=4, m=10$) are strictly dominated and thus never optimal (under assumption~\ref{item:D4}).
A similar figure was also shown in~\cite{gilmer2018}, but with a completely different focus and lack of theoretical foundation. By analysing our advanced adversarial classification game with the help of $\ASR$ and $\CCR$, we justify these figures by a solid theory and encourage researchers to incorporate this evaluation method when reporting results about new attack methods or countermeasures.

Then, as mentioned in Sec.~\ref{sec:BRAdef}, in many cases $\Delta\mu^\DEF \approx 0$. 
This holds especially true, if we consider the setting where the defender can choose between any standard trained model (model $i=1$) and its adversarially trained counterpart (model $i=2$). Since the architecture of both models is identical, they incur the same ongoing costs to operate. 
By Table~\ref{tab:DefenderBestResponse} the defender's best response is determined by the sign of \DCCR. This can be interpreted geometrically; $\DCCR=0$ (Case C), corresponds to an intersection of the functions 
\begin{align*}
    \rho &\mapsto \CCR_{11}(\rho)=(1-\rho)\acc_1 + \rho\ \rob_{11}\\
    \rho &\mapsto \CCR_{21}(\rho)=(1-\rho)\acc_2 + 
    \rho\ \rob_{21}\enspace,
\end{align*}
where $\rho=r_1\rmax\in[0,\rmax]$. Visually, this intersection can be seen in Figure~\ref{fig:Madry}
and Figure~\ref{fig:Free}. 
All $\rho$ values before and after the point of intersection correspond to Case B and A respectively. Finally, if $\rmax=1$, then the condition of all cases $A$, $B$ and $C$ are satisfiable by some $r\in\StrategySetADV$. 

Similarly, if $\mu^\ADV$ is close to zero, as motivated above, the adversary will always attack at her maximum. Thus, in such a situation, the defender will face a proportion $\rmax$ adversarial examples out of all samples. 

Contrarily, the defender could also aim at increasing the value of $\mu^\ADV$. This means, either increasing the ongoing costs (e.\,g., by specifically designed countermeasures), decreasing the positive reward, or increasing the negative reward (e.\,g., by law enforcement and legal frameworks that harm an adversary).


Furthermore, some of our simplifying assumptions might be object to discussion, such as that the adversary has to pay the initial costs, even when she does not attack, that misclassifying adversarial examples costs the defender exactly the same as misclassifying benign samples, that each successful adversarial example gives the same reward to the adversary, or that the ongoing costs of the adversary do not depend on the defender's strategy. Our intention was to advance the theory of adversarial classification games by a first simple cost/reward structure, and thus we leave the mentioned extensions for future work.

Finally, we want to mention that although interesting from a game-theoretical point of view, the situations where adversary and defender play a mixed strategy Nash equilibrium might not be particularly relevant in practice. First of all, as mentioned above, in practice it seems unrealistic to reach one of these states at all, and even if we reach such a state, these Nash equilibria are very unstable. Small changes to the defender's strategy lead to a pure strategy best response of the adversary. Small changes to the adversary's strategy lead to a pure strategy best response of the defender.

\section{Conclusion}
We started this paper with the question: when should you defend your classifier? 

To answer this, we present the \emph{advanced adversarial classification game} that captures all relevant aspects of the interplay between an adversary and a defender in adversarial machine learning. We introduce two new metrics, the attack success rate for the adversary and the correct classification rate for the defender, which enables us to capture both players' expected payoff when being faced with every, possibly mixed, opponent strategy. By analyzing in detail the most common case in the literature, where the adversary has one possible attack and the defender may choose to implement one countermeasure, we are able to identify pure and mixed strategy equilibria for our game. 

By taking into consideration that in realistic scenarios both cost parameters, $\mu^\ADV$ and $\Delta\mu^\DEF$ will be close to zero, we can conclude that the most important parameter of the game is $\rmax$, i.\,e., the proportion of samples an adversary can perturb. As shown in Figures~\ref{fig:Madry} 
and~\ref{fig:Free}, no rational defender would implement any of these proposed countermeasures if she would expect less than $\approx 17\%$
, respectively $\approx 10\%$ 
 adversarial examples. Putting this into an universal answer to the question we set out to answer, it means: 
\begin{center}
Do not defend your classifier when $\rmax \leq \frac{\Dacc}{\Dacc+\Drob}$.
\end{center}

\bibliographystyle{splncs04}
\bibliography{main}

\end{document}